\documentclass{colt2016} % Anonymized submission
\title[Fixed Budget Best Arm Identification]{Tight (Lower) Bounds for the Fixed Budget Best Arm Identification Bandit Problem}
\usepackage{times}
 \coltauthor{\Name{Alexandra Carpentier} \Email{carpentier@uni-potsdam.de}
 \AND
 \Name{Andrea Locatelli} \Email{locatell@uni-potsdam.de}\\
 \addr Department of Mathematics, University of Potsdam, Germany
 }

\begin{document}

\maketitle

\begin{abstract}
We consider the problem of \textit{best arm identification} with a \textit{fixed budget $T$}, in the $K$-armed stochastic bandit setting, with arms distribution defined on $[0,1]$. We prove that any bandit strategy, for at least one bandit problem characterized by a complexity $H$, will misidentify the best arm with probability lower bounded by
$$\exp\Big(-\frac{T}{\log(K)H}\Big),$$
where $H$ is the sum for all sub-optimal arms of the inverse of the squared gaps. Our result disproves formally the general belief - coming from results in the fixed confidence setting - that there must exist an algorithm for this problem whose probability of error is upper bounded by $\exp(-T/H)$. This also proves that some existing strategies based on the Successive Rejection of the arms are optimal - closing therefore the current gap between upper and lower bounds for the fixed budget best arm identification problem.
\end{abstract}

\begin{keywords}
Bandit Theory, Best Arm Identification, Simple Regret, Fixed Confidence Setting, Lower Bounds.\footnote{One of the authors of this paper is a student and we would therefore like to be considered for the best student paper award.}
\end{keywords}

\section{Introduction}

In this paper, we consider the problem of \textit{best arm identification} with a \textit{fixed budget $T$}, in the $K$-armed stochastic bandit setting. Given $K$ distributions (or arms) that take value in $[0,1]$, and given a fixed number of samples $T>0$ (or budget) that can be collected sequentially and adaptively from the distributions, the problem of the learner in this setting is to identify the set of distributions with the highest mean, denoted $\mathcal{A}^*$. This setting was introduced in~\cite{bubeck2009pure,audibert2010best}, and is a variant of the best arm identification problem with \textit{fixed confidence} introduced in~\cite{even2002pac,mannor2004sample}.

The best arm identification problem is an important problem in practice as well as in theory, as it is the simplest setting for stochastic non-convex and discrete optimization. It was therefore extensively studied, see~\cite{even2002pac,mannor2004sample,bubeck2009pure,audibert2010best, gabillon2012best,kalyanakrishnan2012pac, jamieson2014best, jamieson2013lil,karnin2013almost,chen2015optimal} and also the full literature review in Section~\ref{sec:lit} for more references and a presentation of the existing results.

Although this problem has been extensively studied, and the results in the \textit{fixed confidence setting} (see see Section~\ref{sec:lit} for a definition and for a presentation of existing results in this setting) have been refined to a point where the optimality gap between best strategies and known lower bounds is really small, see~\cite{chen2015optimal}, there is to the best of our knowledge a major gap between upper and lower bounds in the fixed budget setting. In order to recall this gap, let us write $\mu_k$ for the means of each of the $K$ distributions, $\mu_{(k)}$ for the mean of the arm that has $k$-th highest mean and $\mu^*$ for the highest of these means. Let us define the quantities 
$H = \sum_{k \not \in \mathcal{A}^*}(\mu^* - \mu_k)^{-2}$ and $H_2 = \sup_{k > |\mathcal{A}^*|}k(\mu^* - \mu_{(k)})^{-2}$.
The tightest known lower bound for the probability of not identifying an arm with highest mean after using the budget $T$ is of order
$$\exp\Big(-\frac{T}{H}\Big),$$
while the tightest known upper bounds corresponding to existing strategies for $K \geq 3$ are either
$$ \exp\Big(-\frac{T}{18a}\Big)~~~~~\mathrm{or}~~~~~\exp\Big(-\frac{T}{2\log(K)H_2}\Big),$$
depending on whether the learner has access to an upper bound $a$ on $H$ (first bound) or not (second bound). %, see~\cite{audibert2010best} for a description of these results. 
%while the tightest known upper bound corresponding to existing strategies that have no information on $H$ is of order% an upper bound $a$ on $H$ is of order 
Since $H_2 \leq H \leq 2\log(K) H_2$, this highlights a gap in the scenario where the learner does not have access to a tight upper bound $a$ on $H$. See~\cite{audibert2010best} for the seminal paper where these state of the art results are proven, and~\cite{gabillon2012best, jamieson2013lil,karnin2013almost,chen2014combinatorial} for papers that propose among other results (generally in the fixed confidence setting) alternative strategies for this fixed budget problem, and~\cite{kaufmann2014complexity} for the lower bound.% - and to the best of our knowledge, the fixed confidence setting bounds results of . % - note that most of the papers on the best arm identification setting consider rather the fixed confidence problem than the fixed budget problem, see Section~\ref{sec:lit} for presentation of the results in the fixed confidence setting.

In this paper, we close this gap, improving the lower bound and proving that the strategies developed in~\cite{audibert2010best} are optimal, in both cases (i.e.~when the learner has access to an upper bound $a$ on $H$ or not). Namely, we prove that there exists no strategy that misidentifies the optimal arm with probability smaller than 
$$\exp\Big(-\frac{T}{a}\Big),$$
uniformly over the problems that have complexity $a$, and that there exists no strategy that misidentifies the optimal arm with probability smaller than
$$\exp\Big(-\frac{T}{\log(K)H}\Big),~~~~~~~~~\Big[\text{and note that}\exp\Big(-\frac{T}{\log(K)H}\Big) \geq \exp\Big(-\frac{T}{\log(K)H_2}\Big)\Big]$$
uniformly over all problems. The first lower bound of order $\exp(-\frac{T}{a})$ is not surprising when one considers the lower bounds results in the \textit{fixed confidence setting} by~\cite{even2002pac,mannor2004sample,gabillon2012best, kalyanakrishnan2012pac, jamieson2014best,jamieson2013lil,karnin2013almost,chen2015optimal}, and was already implied by the results of~\cite{kaufmann2014complexity}, but the second lower bound of order $\exp(-\frac{T}{\log(K)H})$ is on the other hand quite unexpected in light of the results in the fixed confidence setting. In fact it is often informally stated in the fixed confidence literature that since the sample complexity in the fixed confidence setting is $H$, the same should hold for the fixed budget setting, and that therefore the right complexity should be $H$ and not $H\log(K)$, i.e.~it is often conjectured that the right bound should be $\exp(-\frac{T}{H})$ and not $\exp(-\frac{T}{\log(K)H})$. In this paper, we disprove formally this conjecture and prove that in the fixed budget setting, \textit{unlike in the fixed confidence setting}, there is an additional $\log(K$) price to pay for adaptation to $H$ in the absence of knowledge over this quantity. Moreover, our lower bound proofs are very simple, short, and based on ideas that differ from previous results, in the sense that we consider a class of problems with different complexities.

%We would like to emphasize that such a optimality
%gap is not present in the 

In Section~\ref{sec:set}, we present formally the setting, and in Section~\ref{sec:lit}, we present the existing results in a more detailed fashion. Section~\ref{sec:main} contains our main results and Section~\ref{sec:proofg} their proofs.

\section{Setting}\label{sec:set}

\paragraph{Learning setting} We consider a classical $K$ armed stochastic bandit setting with fixed horizon $T$. Let $K>1$ be the number of arms that the learner can choose from. Each of these arms is characterized by a distribution $\nu_k$ that we assume to be defined on $[0,1]$. Let us write $\mu_k$ for its mean.
% a Bernoulli distribution of parameter $p_k$, i.e.~$\nu_k = \mathcal B(p_k)$.
Let $T>0$. We consider the following dynamic game setting with horizon $T$, which is common in the bandit literature. For any time $t \geq 1$ and $t \leq T$, the learner chooses an arm $I_t$ from $\mathbb{A} = \{1,...,K\}$. It receives a noisy reward drawn from the distribution $\nu_{I_t}$ associated to the chosen arm. An adaptive learner bases its decision at time $t$ on the samples observed in the past. At the end of the game $T$, the learner returns an arm
$$\hat k_T \in \{1, \ldots, K\}.$$

\paragraph{Objective} In this paper, we consider the problem of \textit{best arm identification}, i.e.~we consider the learning problem of finding dynamically, in $T$ iterations of the game mentioned earlier, one of the arms with the highest mean. Let us define the set of optimal arms as
$$\mathcal{A}^* = \arg\max_k \mu_k,$$
and $\mu^* = \mu_{k^*}$ with $k^* \in \mathcal{A}^*$ as the highest mean of the problem. Then we define the \textit{expected loss} of the learner as the probability of not identifying an optimal arm, i.e.~as
$$\mathbb P\Big(\hat k_T \not\in \mathcal{A}^*\Big),$$
where $\mathbb P$ is the probability according to the samples collected during the bandit game. The aim of the learner is to follow a strategy that minimizes this expected loss.

This is known as the \textit{best arm identification} problem in the \textit{fixed budget setting}, see~\cite{audibert2010best}. As was explained in~\cite{audibert2010best}, it is linked to the notion of \textit{simple regret}, where the simple regret is the expected sub-optimality of the chosen arm with respect to the highest mean, i.e.~it is $\mathbb E(\mu^* - \mu_{\hat k_T})$, where  $\mathbb E$ is the expectation according to the samples collected during the bandit game.

\paragraph{Problem dependent complexity} We now define two important problem dependent quantities, following e.g.~\cite{even2002pac,mannor2004sample,audibert2010best, gabillon2012best,kalyanakrishnan2012pac, jamieson2014best, jamieson2013lil,karnin2013almost,chen2015optimal}. We will characterize the \textit{complexity} of bandit problems by the quantities 
\begin{equation}\label{eq:comp}
H = \sum_{k \not \in \mathcal{A}^*}\frac{1}{(\mu^* - \mu_k)^2} ~~~~\mathrm{and} ~~~ H_2 = \sup_{k > |\mathcal{A}^*|}\frac{k}{(\mu^* - \mu_{(k)})^2},
\end{equation}
where for any $k \leq K$, $\mu_{(k)}$ is the $k$-th largest mean of the arms. As noted in~\cite{audibert2010best}, the following inequalities hold $H_2 \leq H \leq \log(2K) H_2 \leq 2\log(K) H_2$.

%These two problem dependent quantities were both introduced in~\cite{audibert2010best}, and the results in their paper are based on these quantities.

\section{Literature review}\label{sec:lit}

The problem of \textit{best arm identification} in the $K$ armed stochastic bandit problem has gained wide interest in the recent years. It can be cast in two settings, \textit{fixed confidence}, see~\cite{even2002pac,mannor2004sample}, and \textit{fixed budget}, see~\cite{bubeck2009pure,audibert2010best}, which is the setting we consider in this paper. In the \textit{fixed confidence setting}, the learner is given a precision $\delta$ and aims at returning an optimal arm, while collecting as few samples as possible. In the \textit{fixed budget setting}, the objective of the learner is to minimize the probability of not recommending an optimal arm, given a fixed budget of $T$ pulls of the arms. The links between these two settings are discussed in details in~\cite{gabillon2012best,karnin2013almost}: the fixed confidence setting is a stopping time problem and the fixed budget setting is a problem of optimal resource allocation. It is argued in~\cite{gabillon2012best} that these problems are equivalent. But as noted in~\cite{karnin2013almost,kaufmann2014complexity}, this equivalence holds only if some additional information e.g. $H$ is available in the fixed budget setting, otherwise it appears that the fixed budget setting problem is significantly harder. This fact is highlighted in the literature review below.

\paragraph{Fixed confidence setting} The fixed confidence setting has been more particularly investigated, with papers proposing strategies that are more and more refined and clever. The papers~\cite{even2002pac,mannor2004sample} introduced the problem and proved the first upper and lower bounds for this problem (where $\gtrsim$ and $\lesssim$ are $\geq$ and $\leq$ up to a constant)).
\begin{itemize}
\item Upper bound : There exists an algorithm that returns, after $\hat T$ number of pulls, an arm $\hat k_{\hat T}$ that is optimal with probability larger than $1-\delta$, and is such that the number of pulls $\hat T$ satisfies
$$\mathbb E \hat T \lesssim H \Big(\log(\delta^{-1}) + \log(K) + \log\big((\max_{k\not\in A^*}(\mu^* - \mu_k)^{-1})\big)\Big).$$
\item Lower bound : For any algorithm that returns an arm $\hat k_{\hat T}$ that is optimal with probability larger than $1-\delta$, the number of pulls $\hat T$ satisfies
$$\mathbb E \hat T \gtrsim H \Big(\log(\delta^{-1})\Big).$$
\end{itemize}

These first results already showed that the quantity $H$ plays an important role for the best arm identification problem. These results are tight in the multiplicative terms $H$ but are not tight in the second order logarithmic terms - and there were several interesting works on how to improve both upper and lower bounds to make these terms match, see~\cite{gabillon2012best, kalyanakrishnan2012pac, jamieson2014best, jamieson2013lil,karnin2013almost,kaufmann2014complexity,chen2015optimal}. To the best of our knowledge, the most precise upper bound is in~\cite{chen2015optimal}, and the most precise lower bound in the case of the two armed problem is in~\cite{kaufmann2014complexity}. These bounds, although not exactly matching in general, are matching up to a multiplicative constant for $\delta$ small enough with respect to $H,K$, i.e.~for $\delta$ small enough with respect to $H,K$, it holds that both upper and lower bounds on $\mathbb E \hat T$ are of order
$$H\log(\delta^{-1}).$$
Note that this can already be seen from the two bounds reported in this paper, i.e.~for $\delta$ smaller than $\min\Big(K^{-1}, \max_{k\not\in \mathcal{A}^*}(\mu^* - \mu_k)\Big)$.

\paragraph{Fixed budget setting} The fixed budget has also been studied intensively, but to the best of our knowledge, an important gap still remains between upper and lower bound results. The best known (up to constants) upper bounds are in the paper~\cite{audibert2010best}, while the best lower bound can be found in~\cite{kaufmann2014complexity}, and they are as follows.
\begin{itemize}
\item Upper bound : Assume that an upper bound $a$ on the complexity $H$ of the problem is known to the learner. There exists an algorithm that, at the end of the budget $T$, fails selecting an optimal arm with probability upper bounded as
$$\mathbb P\Big(\hat k_T \not\in \mathcal A^*\Big) \leq 2TK\exp\Big(-\frac{T-K}{18a}\Big).$$
Even if no upper bound on the complexity $H$ is known to the learner, there exists an algorithm that, at the end of the budget $T$, fails selecting an optimal arm with probability upper bounded as
$$\mathbb P\Big(\hat k_T \not\in \mathcal A^*\Big)\leq \frac{K(K-1)}{2}\exp\Big(-\frac{T-K}{\log(2K)H_2}\Big).$$

\item Lower bound : Even if an upper bound on $H, H_2$ is known to the learner, any algorithm, at the end of the budget $T$, fails selecting an optimal arm with probability lower bounded as
$$\mathbb P\Big(\hat k_T \not\in \mathcal A^*\Big)\geq \exp\Big(-\frac{4T}{H}\Big).$$
\end{itemize}

Several papers exhibit other strategies for the fixed budget problem (in general in combination with a fixed confidence strategy), see e.g.~\cite{gabillon2012best, jamieson2013lil,karnin2013almost}, but their theoretical results do not outperform the ones recalled here and coming from~\cite{audibert2010best}. Note that these results highlight a gap between upper and lower bounds. In the case where an upper bound $a$ on the complexity $H$ is known to the learner, the gap is related to the distance between $a$ and $H$. Beyond the fact that $H_2$ is always smaller than $H$, we would like to emphasize here that if the upper bound $a$ on $H$ is not tight enough, the algorithm's performance will be sub-optimal compared to the hypothetical performance of an oracle algorithm that has access to $H$ - as the non-oracle algorithm will over explore. % Moreover, assuming knowledge on $H$ is as strong as assuming some knowledge on the smallest arm gap $\min_k (\mu^* - \mu_k)$, which is not a desirable feature in practice.
Now in the case where one does not want to assume the knowledge of $H$, the gap between known upper and lower bounds becomes even larger and is related to the distance between $H$ and $\log(2K) H_2$. Unlike in the fixed confidence setting, this gap remains also for $T$ large (which corresponds to $\delta$ small in the fixed confidence setting).

We would like to emphasize that although this gap is often belittled in the literature, as it is ``only" a a gap up to a $\log(K)$ factor, this $\log(K)$ factor has an effect in the exponential, and in some sense it is much larger than the gap that was remaining in the fixed confidence setting after the seminal papers~\cite{even2002pac,mannor2004sample}, and over which many valuable works have further improved. Indeed, in order to compare the bounds in the fixed confidence setting with the bounds in the fixed budget setting, one can set
$\delta := \mathbb P\Big(\hat k_T \not\in \mathcal A^*\Big)$, and compute the fixed budget $T$ for which a precision of at least $\delta$ is achieved for both upper and lower bounds. Inverting the upper bounds in the fixed budget setting, one would get the upper bounds on $T$
$$T \lesssim a\log(KT/\delta),~~~~\mathrm{or} ~~~ T \lesssim H_2\log(K)\log(K/\delta)),$$
when respectively an upper bound $a$ on $H$ is known by the learner or when no knowledge of $H$ is available. Conversely, the lower bound in the fixed budget setting yields that the fixed budget $T$ must be of order higher than
$$T \gtrsim H\log(1/\delta).$$

As mentioned, this gap also remains for $\delta$ small. This highlights the fact that the gap in the fixed budget setting is much more acute than the gap in the fixed confidence setting, and that this $\log(K)$ factor is not negligible if one looks at the fixed budget setting problem from the fixed confidence setting perspective. This knowledge gap between the fixed confidence and fixed budget setting was underlined in the papers~\cite{karnin2013almost,kaufmann2014complexity} where the authors explain that closing the gap in the fixed budget setting is a difficult problem that goes beyond known techniques for the fixed confidence setting.

We close this review of literature by mentioning related works on the more involved TopK bandit problem, where the aim is to find $k$ arms that have the highest means, see~\cite{bubeckmultiple,gabillon2012best, kaufmann2014complexity, zhou2014optimal,cao2015top}, and also the more general pure exploration bandit setting introduced in~\cite{chen2014combinatorial}. These results apply to the best arm identification problem considered in this paper, which is a special case of their settings, but they do not improve on the mentioned results for the best arm identification problem.

\section{Main results} \label{sec:main}

We state our results in two parts. First, we provide a weaker version of our results in Subsection~\ref{ss:wr}, which has the advantage of not requiring the introduction of too many additional technical notations We then propose in Subsection~\ref{ss:tf} a technical and stronger formulation of our results.

\subsection{First formulation of our results}\label{ss:wr}

We state the following lower bound for the bandit problem introduced in Section~\ref{sec:set}.
\begin{theorem}\label{thm:main3}
Let $K>1$, $a>0$. Let $\mathbb B_a$ be the set of all bandit problems with distributions in $[0,1]$ and complexity $H$ bounded by $a$. For $\mathcal G \in \mathbb B_a$, we write $\mathcal A^*(\mathcal G)$ for the set of arms with highest mean of problem $\mathcal G$, and $H(\mathcal G)$ for the complexity defined in Equation~\eqref{eq:comp} as $H$ (first quantity) and associated to problem $\mathcal G$.

If $T\geq a^2 \big(4\log(6TK)\big)/ (60)^2$, for any bandit strategy that returns arm $\hat k_T$ at time $T$, it holds that
$$\sup_{\mathcal G \in \mathbb B(a)} \mathbb P_{\mathcal G^{\otimes T}} (\hat k_T \not \in \mathcal A^*(\mathcal G)) \geq \frac{1}{6}\exp\Big(-120\frac{T}{a} \Big).$$

If in addition $a \geq 11K^2$ and if $K \geq 2$, then for any bandit strategy that returns arm $\hat k_T$ at time $T$, it holds that
$$\sup_{\mathcal G \in \mathbb B(a)} \Bigg[\mathbb P_{\mathcal G^{\otimes T}} (\hat k_T \not \in \mathcal A^*(\mathcal G)) \times \exp\Big(400\frac{T}{\log(K)H(\mathcal G)} \Big)\Bigg]\geq  \frac{1}{6}.$$
\end{theorem}

This theorem implies what we described in the introduction:
\begin{itemize}
\item Even when an upper bound $a$ on the complexity $H$ of the target bandit problem is known, any learner will misidentify the arm with highest mean with probability larger than
$$ \frac{1}{6}\exp\Big(-120\frac{T}{a} \Big),$$
on at least one of the bandit problems with complexity $H$ bounded by $a$.
\item For $T, a,K$ large enough - $T$ of larger order than $a^2\log(K)$, $a$ of larger order than $K^2$ and $K$ larger than $2$ - any learner will misidentify the arm with highest mean with probability larger than
$$ \frac{1}{6}\exp\Big(-400\frac{T}{\log(K)H(\mathcal G)} \Big),$$
on at least one of the bandit problems $\mathcal G\in \mathbb B_a$ which is associated to some complexity $H(\mathcal G)$ bounded by $a$.
\end{itemize}
The first result is expected when one looks at the lower bounds in the fixed confidence setting, see~\cite{even2002pac,mannor2004sample,gabillon2012best, kalyanakrishnan2012pac, jamieson2014best,jamieson2013lil,karnin2013almost,kaufmann2014complexity,chen2015optimal}. On the other hand, the second result cannot be conjectured from lower bounds in the fixed confidence setting. We remind that in order to obtain a precision $\delta>0$ in the fixed confidence setting, even if the learner does not know $H$, it only requires $$O(H\log(\delta^{-1})),$$ samples for $\delta$ small enough. The natural conjecture following from this is that the probability of error in the fixed budget setting is $$\exp(-T/H),$$ for $T$ large enough. We proved that this does not hold and that the probability of error in the fixed budget setting is lower bounded for any strategy in at least one problem by $$\exp(-T/(\log(K)H)),$$ for $T$ large enough - which corresponds to a higher sample complexity $$H \log(K)\log(1/\delta),$$ in the fixed confidence setting. This lower bound highlights a fundamental difference between the fixed confidence setting - where one does not need to know $H$ in order to adapt to it - and the fixed budget setting - where in the absence of the knowledge of $H$, one pays a price of $\log(K)$ for the adaptation. % We want to highlight that in the fixed confidence setting, in order to obtain a precision $\delta$ and even if where even if the learner does not know $H$, it only requires $O(H\log(\delta^{-1}))$ samples, as opposed to $O(H\log(K)\log(\delta^{-1}))$ samples
Moreover, this lower bound proves that the Successive Reject strategy introduced in~\cite{audibert2010best} is optimal, as its probability of error is upper bounded by a quantity of order
$$\exp(-T/(\log(K)H_2)),$$
which is always smaller in order than our lower bound of order%the order of magnitude of our lower bound
$$\exp(-T/(\log(K)H)).$$
This might seem contradictory as the lower bound might seem higher than the upper bound. It is of course not and this only highlights that the problems on which all strategies won't perform well are problems such that $H_2$ is of same order as $H$ - problems having many sub-optimal arms close to the optimal ones. These problems are the most difficult problems in the sense of adapting to the complexity $H$, and for them, a $\log(K)$ adaptation price is unavoidable. This kind of phenomenon, i.e.~the necessity of paying a price for not knowing the model (here the complexity $H$), is not very much studied in the bandit literature, but arises in many fields of high dimensional statistics and non-parametric statistics, see e.g.~\cite{lepski1997optimal, bunea2007sparsity}.

\subsection{Technical and stronger formulation of the results}\label{ss:tf}

We will now present the technical version of our results. This is a lower bound that will hold in the much easier (for the learner) problem where the learner knows that the bandit setting it is facing is one of only $K$ given bandit settings (and where it has all information about these settings). This lower bound ensures that even in this much simpler case, the learner, however good it is, will nevertheless make a mistake.

Before stating the main technical theorem, let us introduce some notations about these $K$ settings. Let $(p_k)_{2 \leq k\leq K}$ be $(K-1)$ real numbers in $[1/4, 1/2)$. Let $p_1 = 1/2$. Let us write for any $1 \leq k\leq K$, $\nu_k:=\mathcal B(p_k)$ for the Bernoulli distribution of mean $p_k$, and $\nu_k':=\mathcal B(1-p_k)$ for the Bernoulli distribution of mean $1-p_k$.

We define the product distributions $\mathcal{G}^i$ where $i \in \{1, ..., K\}$ as $\nu_1^i \otimes ... \otimes \nu_K^i$ where for $1 \leq k \leq K$, $$\nu_k^i := \nu_i \mathbf 1\{k \neq i\} + \nu_i' \mathbf 1\{k = i\}.$$
The bandit problem associated with distribution $\mathcal G^i$, and that we call ``the bandit problem $i$" is such that for any $1 \leq k \leq K$, arm $k$ has distribution $\nu_k^i$, i.e.~all arms have distribution $\nu_k$ except arm $i$ that has distribution $\nu_i'$. We write for any $1 \leq i \leq K$, $\mathbb P_{i}:=\mathbb P_{(\mathcal G^i)^{\otimes T}}$ for the probability distribution of the bandit problem $i$ according to all the samples that a strategy could possibly collect up to horizon $T$, i.e.~according to the samples $(X_{k,s})_{1 \leq k \leq K, 1 \leq s \leq T} \sim (\mathcal{G}^i)^{\otimes T}$.

% We also extend this notation to $\mathcal{G}^0$, where none of the arms is flipped with respect to arm $0$ ($\forall k$, $\nu_k^0:= \nu_i$).\\

We define for any $1 \leq k \leq K$ the quantities $d_k := 1/2 - p_k$. 
%Note that if the distribution of the arms of the bandit problem is $\mathcal{G}^0$, then the $(d_k)_{k\leq K}$ correspond to the gaps with respect to the best arm (since its values is $1/2$ in this problem).
Set also for any $i \in \{1, ..., K\}$ and any $k \in \{1, ..., K\}$
$$\Delta_k^i = d_i + d_k,~~~\mathrm{if}~~ k \neq i~~~~~~\mathrm{and}~~~~~~\Delta_i^i = d_i.$$
In the bandit problem $i$, as \textit{the arm with the best mean is $i$} (and its mean is $1-p_i = 1/2 + d_i$), one can easily see that the $(\Delta_k^i)_k$ are the arm gaps of the bandit problem $i$.

We also define for any $1 \leq i \leq K$ the quantity%t $H(i)$ as:
$$H(i) := \sum_{1 \leq k \leq K, k \neq i} (\Delta_k^i)^{-2},$$%~~~~~~~~%\mathrm{and}~~~~~~~~H_2(i) := \max_{1 \leq k \leq K, k \neq i} k(\Delta_{(k)}^i)^{-2},$$
%where $\Delta_{(k)}^i$ is the $k$-th smallest element among the $(\Delta_{k}^i)_k$.
with $H(1) = \max_{1\leq i \leq K}H(i)$. The quantities $H(i)$ correspond to the complexity $H$ computed for the bandit problem $i$ and introduced in Equation~\eqref{eq:comp} (first quantity). We finally define the quantity
$$h^* = \sum_{K \geq k \geq 2} \frac{1}{d_i^{2} H(i) }.$$
%\max(\frac{2}{\Delta_i^2}, \max\limits_{k<i} \frac{k+2}{(\Delta_k + \Delta_i)^2}, \max\limits_{k>i} \frac{k+1}{(\Delta_k + \Delta_i)^2}).$$

We can now state our main technical theorem - \textit{we remind that there is only one arm with highest mean in the bandit problem $i$, and that this arm is arm $i$, so $\mathbb P_{i} (\hat k_T \neq i)$ is the probability under bandit $i$ of not identifying the best arm and recommending a sub-optimal arm.}
\begin{theorem}\label{thm:main}
For any bandit strategy that returns the arm $\hat k_T$ at time $T$, it holds that
$$\max_{1 \leq i \leq K} \mathbb P_{i} (\hat k_T \neq i) \geq  \frac{1}{6}\exp\Big(-60\frac{T}{H(1)} -2 \ \sqrt[]{T\log(6TK)}\Big),$$
where we remind that $H(1) = \max_i H(i)$ and also
$$\max_{1 \leq i \leq K}\Bigg[ \mathbb P_{i} (\hat k_T \neq i) \times \exp\Big(60\frac{T}{{H(i)}h^*} 
+ 2 \ \sqrt[]{T\log(6TK)}\Big) \Bigg]\geq 1/6.$$
\end{theorem}
The proof of this result is different from the proof of other lower bounds for best arm identification in the fixed budget setting as in~\cite{audibert2010best}. Its construction is not based on a permutation of the arms, but on a flipping of each arm around the second best arm - see Subsection~\ref{sec:proof}. A similar construction can be found in~\cite{kaufmann2014complexity}. However, similarly to~\cite{audibert2010best}, in this paper, a single complexity $H$ is used in the proof, while our proof involves a range of complexities. The idea of the proof is that for any bandit strategy there is at least one bandit problem $i$ among the $K$ described where an arm will be pulled less than it should according to the optimal allocation of the problem $i$ - and when this happens, the algorithm makes a mistake with probability that is too high with respect to the complexity $H(i)$ of the problem. This Theorem is a stronger version of Theorem~\ref{thm:main3} since it states than even if the learner knows that the bandit problem he faces is one of $K$ problems fully described to him, he will nevertheless make an error with probability lower bounded by problem dependent quantities that are much larger than the ones in~\cite{audibert2010best,kaufmann2014complexity}.

A version of this theorem that is easier to read and that holds for $T$ large enough, is as follows.
\begin{corollary}\label{thm:main2}
Assume that $T\geq \max\Big(H(1), H(i)h^*\Big)^2 4\log(6TK)/ (60)^2$.  For any bandit strategy that returns the arm $\hat k_T$ at time $T$, it holds that
$$\max_{1 \leq i \leq K} \mathbb P_{i} (\hat k_T \neq i) \geq  \frac{1}{6}\exp\Big(-120\frac{T}{H(1)} \Big) = \frac{1}{6}\exp\Big(-120\frac{T}{\max_iH(i)} \Big),$$
and also
$$\max_{1 \leq i \leq K}\Bigg[ \mathbb P_{i} (\hat k_T \neq i) \times \exp\Big(120\frac{T}{{H(i)}h^*} \Big) \Bigg]\geq 1/6.$$
\end{corollary}
Note that both Theorems~\ref{thm:main} and Corollary~\ref{thm:main2} hold for any $p_2, \ldots, p_k$ that belong to $[1/4, 1/2)$ and are therefore quite general.

%\paragraph{Remark of the results}

\section{Proof of the theorems}\label{sec:proofg}

\subsection{Proof of Theorem~\ref{thm:main}}\label{sec:proof}

\paragraph{Step 1: Definition of a high probability event where empirical KL divergences concentrate} For two distributions $\nu,\nu'$ defined on $\mathbb R$ and that are such that $\nu$ is absolutely continuous with respect to $\nu'$, we write
$$\text{KL}(\nu,\nu') = \int_{\mathbb R} \log\Big(\frac{d\nu(x)}{d\nu'(x)}\Big)d\nu(x),$$
for the Kullback leibler divergence between distribution $\nu$ and $\nu'$.

Let $k \in \{1, ..., K\}$. Let us write
$$ \text{KL}_k := \text{KL}(\nu_k', \nu_k) = \text{KL}(\nu_k, \nu_k') = (1 - 2p_k)\log\big(\frac{1 - p_k}{p_k}\big),$$
for the Kullback-Leibler divergence between two Bernoulli distributions $\nu_k$ and $\nu_k'$ of parameter $p_k$ and $1-p_k$. Since $p_k \in [1/4,1/2)$, the following inequality holds:
\begin{equation}\label{eq:gapKL}
\text{KL}_k \leq 10d_k^2.
\end{equation}

Let $1\leq t\leq T$. We define the quantity:
\begin{align*}
\widehat{\text{KL}}_{k,t} &= \frac{1}{t} \sum_{s=1}^t \log(\frac{d \nu_k}{d \nu_k'}(X_{k,s}))\\
&= \frac{1}{t} \sum_{s=1}^t \mathbf 1\{X_{k,s} = 1\} \log(\frac{p_i}{1-p_i}) + \mathbf 1\{X_{k,s} = 0\}\log(\frac{1-p_i}{p_i}),
\end{align*}
where by definition for any $s \leq t$, $X_{k,s}\sim_{i.i.d} \nu_k^i$.

Let us define the event
\begin{align*}
\xi &= \Big\{\forall 1 \leq k\leq K, \forall 1 \leq t\leq T, |\widehat{\text{KL}}_{k,t}| - \text{KL}_k\leq 2 \ \sqrt[]{\frac{\log(6TK)}{t}} \Big\}.
\end{align*}
We now state the following lemma, i.e.~a concentration bound for $|\widehat{\text{KL}}_{k,t}|$ that holds for all bandit $i$ with $1\leq i\leq K$.
\begin{lemma}\label{xi}
%Let us consider bandit $i$ for $1 \leq i \leq K$. Then
It holds that
$$\mathbb P_{i}(\xi) \geq 5/6.$$
\end{lemma}
\begin{proof}
If $k \neq i$ (and thus $\nu_k^i = \nu_k$) then $\mathbb E_{\mathcal G^i} \widehat{\text{KL}}_{k,t} = \mathrm{KL}_k$ and if $k = i$ (and thus $\nu_k^i = \nu_k'$) then $\mathbb E_{\mathcal G^i} \widehat{\text{KL}}_{k,t} = -\mathrm{KL}_k$. Moreover note that since $p_k \in [1/4,1/2)$
$$|\log(\frac{d \nu_k}{d \nu_k'}(X_{k,s}))| = |\mathbf 1\{X_{k,s} = 1\} \log(\frac{p_i}{1-p_i}) + \mathbf 1\{X_{k,s} = 0\}\log(\frac{1-p_i}{p_i})| \leq \log(3).$$
%If $i\neq k$, this is an empirical estimator of $\text{KL}(\nu_k, \nu_k')$ after $T$ pulls, otherwise this is an empirical estimator of $-\text{KL}(\nu_k, \nu_k')$ after $T$ pulls.\\
Therefore, $\widehat{\text{KL}}_{k,t}$ is a sum of i.i.d.~samples that are bounded by $\log(3)$, and whose mean is $\pm \mathrm{KL}_k$ depending on the value of $i$. We can apply Hoeffding's inequality to this quantity and we have that with probability larger than $1-(6KT)^{-1}$
$$|\widehat{\text{KL}}_{k,t}| - \text{KL}_k\leq \sqrt{2}\log(3) \ \sqrt[]{\frac{\log(6TK)}{t}}.$$
This assertion and an union bound over all $1 \leq k \leq K$ and $1 \leq t \leq T$ implies that $\mathbb P_{\mathcal G^i}(\xi) \geq 5/6$, as we have $\sqrt{2}\log(3) < 2$.
%Since $\widehat{\text{KL}}_{k,t}$ is bounded with mean $\text{KL}_k$, using Hoeffding's maximal inequality, it holds that for any $i$ that $\mathbb P_{\mathcal G^i}(\xi) \geq 5/6$.
\end{proof}

\paragraph{Step 2: A change of measure}

Let now $\mathcal{A}lg$ denote the active strategy of the learner, that returns some arm $\hat k_T$ at the end of the budget $T$. Let $(T_k)_{1 \leq k \leq K}$ denote the numbers of samples collected by $\mathcal{A}lg$ on each arm of the bandits. These quantities are stochastic but it holds that $\sum_{1 \leq k \leq K} T_k =T$ by definition of the fixed budget setting. Let us write for any $0 \leq k \leq K$
$$t_k = \mathbb E_{1} T_k.$$
It holds also that $\sum_{1 \leq k \leq K} t_k =T$ 

We recall the change of measure identity (see e.g.~\cite{audibert2010best}) which states that for any measurable event $\mathcal{E}$ and for any $2 \leq i \leq K$ :
\begin{equation}\label{cm}
\mathbb{P}_{i}(\mathcal{E}) = \mathbb{E}_{1}\Big[\mathbf{1}\{\mathcal{E}\}\exp\big(-T_i\widehat{\text{KL}}_{i,T_i}\big)\Big],
\end{equation}
as the product distributions $\mathcal G^i$ and $\mathcal G^1$ only differ in $i$ and as the active strategy only explored the samples $(X_{k,s})_{k \leq K, s \leq T_k}$.

Let $2\leq i \leq K$. Consider now the event 
$$\mathcal{E}_i = \{\hat k_T = 1\} \cap \{\xi\} \cap \{T_i \leq 6 t_i\},$$
i.e.~the event where the algorithm outputs arm $1$ at the end, where $\xi$ holds, and where the number of times arm $i$ was pulled is smaller than $6t_i$. We have by Equation~\eqref{cm} that
\begin{align}
\mathbb{P}_{i}(\mathcal{E}_i) & =\mathbb{E}_{1}\Big[\mathbf{1}\{\mathcal{E}_i\}\exp\big(-T_i\widehat{\text{KL}}_{i,T_i}\big)\Big]\nonumber\\
%& \geq \mathbb{E}_{\mathcal{G}^0}\Big[\mathbf{1}_{\mathcal{A}_i\cap \xi}\exp\big(-T_i\widehat{\text{KL}}_{i,T_i}\big)\Big]\\
%& \geq \mathbb{E}_{\mathcal{B}^0}\Big[\mathbf{1}_{\mathcal{A}_i\cap \xi}\exp\big(-T_i\text{KL}_i\\
%&\qquad - \Delta_i  \sqrt[]{T_i\log(\log(T)K))}\big)\Big]\\
& \geq \mathbb{E}_{1}\Big[\mathbf{1}\{\mathcal E_i\}\exp\Big(-T_i\text{KL}_i -2 \ \sqrt[]{T_i\log(6TK)}\Big)\Big] \nonumber\\
& \geq \mathbb{E}_{1}\Big[\mathbf{1}\{\mathcal E_i\}\exp\Big(-6t_i\text{KL}_i -2 \ \sqrt[]{T\log(6TK)}\Big)\Big]\nonumber\\
& \geq \exp\Big(-6t_i\text{KL}_i -2 \ \sqrt[]{T\log(6TK)}\Big) \mathbb{P}_{1}(\mathcal E_i),\label{eq:event}
%& \geq \mathbb{E}_{\mathcal{G}^0}\Big[\mathbf{1}_{\mathcal{A}_i\cap \xi \cap \{T_i \leq 6t_i\}}\exp\big(-54\Delta_i^2 t_i - \sqrt[]{T\log(6(\log(T)+1)K)}\big)\Big]\\
%& \geq \mathbb{P}_{\mathcal{G}^0}(\mathcal{A}_i\cap \xi \cap \{T_i \leq 6t_i\})\exp\big(-54\Delta_i^2 t_i - \sqrt[]{T\log(6(\log(T)+1)K)}\big),
\end{align}
since on $\mathcal E_i$, we have that $\xi$ holds and that $T_i \leq 6t_i$, and since $\mathbb E_{1} \widehat{\text{KL}}_{i,t} = \text{KL}_i$ for any $t \leq T$.% definition of $\xi$ and $\text{KL}_i$, with $t_i = \mathbb{E}_{\mathcal{G}^0}[T_i]$.
%where the second line uses $\{\mathcal{A}_i \cap \xi\} \subset \mathcal{A}_i$. %In the following, we write $C := C(T,K) = 2\log(3) \ \sqrt[]{\frac{\log((4TK)/3)T}{2}}$.\\

\paragraph{Step 3 : Lower bound on $\mathbb{P}_{1}(\mathcal E_i)$ for any reasonable algorithm} Assume that for the algorithm $\mathcal Alg$ that we consider
\begin{equation}\label{eq:bras0}
\mathbb E_{1}(\hat k_T \neq 1) \leq 1/2,
\end{equation}
i.e.~that the probability that $\mathcal Alg$ makes a mistake on problem $1$ is less than $1/2$. Note that if $\mathcal Alg$ does not satisfy that, it performs badly on problem $1$ and its probability of success is not larger than $1/2$ uniformly on the $K$ bandit problems we defined.

For any $2 \leq k \leq K$ it holds by Markov's inequality that
\begin{align}\label{eq:ma}
\mathbb P_{1} (T_k \geq 6 t_k) \leq \frac{\mathbb E_1 T_k}{6t_k} = 1/6,
\end{align}
since $\mathbb E_{1} T_k = t_k$ for algorithm $\mathcal Alg$,

So by combining Equations~\eqref{eq:bras0},~\eqref{eq:ma} and Lemma~\ref{xi}, it holds by an union bound that for any $2 \leq i \leq K$
$$\mathbb P_{1}(\mathcal E_i) \geq 1 - (1/6+1/2 +1/6) = 1/6.$$
This fact combined with Equation~\eqref{eq:event} and the fact that for any $2 \leq i \leq K$ $\mathbb{P}_{i}(\hat k_T \neq i)  \geq \mathbb{P}_{i}(\mathcal{E}_i)$ implies that for any $2 \leq i \leq K$
\begin{align}
\mathbb{P}_{i}(\hat k_T \neq i)  &\geq \frac{1}{6}\exp\Big(-6t_i\text{KL}_i -2 \ \ \sqrt[]{T\log(6TK)}\Big)\nonumber\\ 
&\geq  \frac{1}{6}\exp\Big(-60t_id_i^2 -2 \ \ \sqrt[]{T\log(6TK)}\Big),\label{eq:event2}
\end{align}
where we use Equation~\eqref{eq:gapKL} for the last step.

\paragraph{Step 4 : Conclusions.}

Since $\sum_{2 \leq k \leq K} d_k^{-2} = H(1)$, and since $\sum_{1 \leq k \leq K} t_k = T$, then there exists $2 \leq i \leq K$ such that
$$t_i \leq \frac{T}{H(1)d_i^2},$$
as the contraposition yields an immediate contradiction.
For this $i$, it holds by Equation~\eqref{eq:event2} that
\begin{align*}
\mathbb{P}_{i}(\hat k_T \neq i)  \geq\frac{1}{6}\exp\Big(-60\frac{T}{H(0)} - 2 \ \sqrt[]{T\log(6TK)}\Big).
\end{align*}
This concludes the proof of the first part of the theorem (note that $H(1) = \max_i H(i)$).

Since $h^* = \sum_{2 \leq k \leq K} \frac{1}{d_k^{2} H(k) }$ and since $\sum_{1 \leq k \leq K} t_k = T$, then there exists $2 \leq i \leq K$ such that
$$t_i \leq \frac{T}{h^* d_i^{2} H(i)}.$$

For this $i$, it holds by Equation~\eqref{eq:event2} that
\begin{equation}
\mathbb{P}_{i}(\hat k_T \neq i)  \geq \frac{1}{6}\exp\Big(-\frac{60T}{h^*H(i)} -2 \ \sqrt[]{T\log(6TK)}\Big).\nonumber\\
\end{equation}

This concludes the proof of the second part of the theorem.

\subsection{Proof of Theorem~\ref{thm:main3}}\label{sec:proof2}

The proof of the first equation in this theorem follows immediately from Corollary~\ref{thm:main2} since $H(1) = \max_i H(i)$.

The proof of the first equation in this theorem follows as well from Corollary~\ref{thm:main2} by taking $d_k = \frac{1}{4}(k/K)$ for $k \geq 2$ (and therefore $p_k = 1/2 - \frac{1}{4}(k/K) \in [1/4, 1/2)$). Note first that this problem belongs to $\mathbb B_a$ with $a = 11K^2$, since $H(i) \leq H(1) \leq 11K^2$. In this case, for any $1\leq i\leq K$, we have 
$$d_i^2H(i) = d_i^2 \sum_{k\neq i} \frac{1}{(d_i+d_k)^2} \leq d_i^2\Big( \frac{i}{d_i^2} + \sum_{k> i} \frac{1}{d_k^2}  \Big) \leq i + i^2\sum_{K \geq k \geq i} \frac{1}{k^2} \leq i +i^2 (\frac{1}{i} - \frac{1}{K})\leq  2i.$$
This implies that
$$h^* \geq \sum_{k = 2}^K \frac{1}{2i}\geq \frac{1}{2} (\log(K+1)-\log(2)) \geq \frac{3}{10}\log(K).$$
This concludes the proof.
\section{An $\alpha-$parametrization}

Building on the ideas exposed in the very last part of the proof, we now consider $d_k^\alpha = \frac{1}{4}(k/K)^\alpha$ for $k \geq 2$, $\alpha \geq 0$. A such construction was already considered for the fixed confidence setting in~\cite{jamieson2013finding}. First, let us state that for any $\alpha$, we have the following inequalities: $H(1) \geq H(i) \geq H(K)$, with $H(K)$ (the easiest problem) of order $K$ for all $\alpha$. The hardest problem on the other hand, has complexity of order
$$
H(1) \simeq 
\left\{ 
\begin{array}{ll}
\frac{1}{1-2\alpha}K, \text{for} \ \alpha < 1/2 \\
\log(K)K, \text{for} \ \alpha = 1/2\\
\frac{1}{2\alpha-1}K^{2\alpha}, \text{for} \ \alpha > 1/2
    \end{array}
\right. .
$$
For $\alpha < 1/2$, both the easiest and hardest problems in our restricted problem class have a similar complexity up to a constant. On the other hand, for $\alpha >1/2$, we have $H(1)$ of order $H(K)^{2\alpha}$, spanning a range of problems with varying complexities. One can easily check that for $\alpha > 1/2$, we have $h^*$ of order at least $\log(K)$ (as we did for $\alpha = 1$ in the previous section). On the other hand, for $\alpha < 1/2$, we can upper bound $h^*$ as follows:
$$
h^* = \sum_{i=2}^K \frac{1}{d_i^2H(i)} \leq \frac{1}{H(K)}\sum_{i=2}^K \frac{1}{d_i^2} = \frac{H(i)}{H(K)},\\
$$
and this ratio is upper bounded by a constant, as both terms are of order $K$. As such, this construction does not imply that a $\log(K)$ adaptation price is unavoidable in all cases, and the question remains open on whether there exists an algorithm that can effectively adapt to these easier problems.

\section*{Conclusion}
In this paper, our main result states that for the problem of best arm identification in the fixed budget setting, if one does not want to assume too tight bounds on the complexity $H$ of the bandit problem, then any bandit strategy makes an error on some bandit problem $\mathcal G$ of complexity $H(\mathcal G)$ with probability at least of order
$$\exp(-\frac{T}{\log(K)H(\mathcal G)}).$$
This result formally disproves the general belief (coming from results in the fixed confidence setting) that there must exist an algorithm for this problem that, for any problem of complexity $H$, makes an error of at most $$\exp(-\frac{T}{H}).$$
This highlights the interesting fact that for this fixed budget problem and \textit{unlike what holds in the fixed confidence setting}, there is a price to pay for adaptation to the problem complexity $H$. This kind of ``adaptation price phenomenon" can be observed in many model selection problems as e.g.~sparse regression, functional estimation, etc, see~\cite{lepski1997optimal, bunea2007sparsity} for illustrations in these settings where such a phenomenon is well known. % other s
This also proves that strategies based on the Successive Rejection of the arms as the Successive Reject of~\cite{audibert2010best}, are optimal. Our proofs are simple and we believe that our result is an important one, since this closes a gap that had been open since the introduction of the fixed confidence best arm identification problem by~\cite{audibert2010best}.

\paragraph{Acknowledgement} This work is supported by the DFG's Emmy Noether grant MuSyAD (CA 1488/1-1).

\bibliography{library}

%\appendix

%\section{My Proof of Theorem 1}

%This is a boring technical proof.

%\section{My Proof of Theorem 2}

%This is a complete version of a proof sketched in the main text.

\end{document}